\documentclass[11pt]{article}
\usepackage[margin=1in]{geometry}
\usepackage{amsmath,amssymb,amsthm,graphicx,booktabs}
\usepackage{tikz}
\usepackage{multirow}
\usepackage[numbers,sort&compress]{natbib}
\usepackage[colorlinks=true,citecolor=blue,linkcolor=blue]{hyperref}
\usepackage{caption}
\newtheorem{proposition}{Proposition}
\newcommand{\R}{\mathbb{R}}
\newcommand{\E}{\mathbb{E}}

\title{The Quadrilateral Loss: Additivity as a Measurable\\ Behavior of Dense Neural Networks}
\author{Antonio Di Cecco\thanks{Universit\`a ``G.~d'Annunzio'' Chieti--Pescara, and School of AI. Author of \emph{Explainable AI with Python} \citep{dicecco2025xai}.}}
\date{July 2026}

\begin{document}
\maketitle

\begin{abstract}
Additive models buy interpretability by forbidding feature interactions, a constraint that neural instantiations enforce architecturally. We introduce the \emph{quadrilateral loss}, a differentiable penalty that treats additivity as a measurable behavior instead: a second-order mixed difference on pairs of training points swapping one coordinate, which vanishes if and only if the coordinate carries no interaction, remains informative for piecewise-linear networks, and equals in expectation the per-coordinate interaction mass of the interventional Shapley-GAM. The loss turns additivity into a dial---most learned interactions prove removable almost for free, and on small datasets a moderate penalty improves accuracy and additivity simultaneously---and into an online observable: its per-feature \emph{surrender curves} show, across seeds and datasets, that pre-regularization interaction magnitude barely predicts what a regularized model retains, undermining post-hoc interaction rankings. Against this instrument we compare routes to exact additivity, spanning structural masks, behavioral penalties (optionally crystallized into exact structure), weight decay, backfitting, the shared-section model, and bagged boosted stumps: constraining behavior before structure dominates weight-space constraints, rankings reverse between data regimes, and converging routes agree on the shape functions themselves. Three silent failure modes we document share one anatomy: guarantees imported into settings that quietly void their preconditions.
\end{abstract}

\section{Introduction}
\label{sec:intro}

Additive models of the form $f(x) = \beta_0 + \sum_{i=1}^{d} f_i(x_i)$ occupy a privileged point on the accuracy--interpretability spectrum: each coordinate's contribution is a one-dimensional function that can be plotted, audited, and constrained, while the component functions themselves remain arbitrarily nonlinear \citep{hastie1986gam}. Neural Additive Models \citep{agarwal2021nam} realize this class with one subnetwork per feature, and their extensions with selected pairwise terms \citep{lou2013ga2m, chang2022node} recover much of the accuracy lost to the additivity constraint. In all of these approaches the constraint is \emph{structural}: the architecture makes interactions impossible by construction.

Structural enforcement is clean but rigid: it commits to the additive hypothesis before training, discards any pretrained dense network, and offers no middle ground between ``fully additive'' and ``fully connected.'' Yet it is far from the only road. One can penalize a dense network's interactions behaviorally and crystallize the structure afterwards; decay its cross-feature weights during training; run the classical backfitting algorithm with neural smoothers; read an additive predictor directly out of a dense network's mean-conditioned sections; or leave neural networks aside and bag depth-one boosted trees. All of these produce (exactly or asymptotically) the same functional class. The question this paper asks is the one a practitioner actually faces: \emph{among the many routes to an additive model, which ones work, when, and do they arrive at the same functions?}

Answering that question requires an instrument before it requires a method: a way to \emph{measure} how interacting a network is. The natural functional measure faces an obstacle that is easy to underestimate. The mixed partial derivative is identically zero almost everywhere for piecewise-linear networks---their interactions live in the gradient's discontinuities, invisible to any pointwise test---so a useful measure must be \emph{nonlocal}. We use the elementary second-order finite difference on a quadrilateral of inputs, which detects interactions across activation boundaries and costs four forward passes; the same statistic serves throughout as diagnostic, as training penalty, and as the criterion by which one-dimensional response plots earn the right to be called shape functions (Section~\ref{sec:method}).

Our contributions, in the order the paper develops them: (i) the quadrilateral loss, a new differentiable penalty, with a proof that it characterizes additivity on product domains (Proposition~\ref{prop:char}), survives piecewise linearity where mixed-derivative penalties fail, and---via a two-line M\"obius argument---equals, in expectation, the total interaction mass of the interventional Shapley-GAM containing each coordinate, so that driving it down also certifies the faithfulness of ordinary SHAP explanations (Section~\ref{sec:method}); (ii) the $\lambda$-sweep, showing a strongly asymmetric accuracy--interaction trade-off in which most interactions are removable almost for free (Figure~\ref{fig:sweep}), together with a regularization sweet spot on small data---seed-universal on Wine---where moderate $\lambda$ improves accuracy and additivity simultaneously, a dimensionless-$\lambda$ argument and balance heuristic $\lambda_{\mathrm{bal}}$ for choosing it, and a tuning experiment showing why $\lambda$ should \emph{not} be selected on accuracy alone (Sections~\ref{sec:sweep} and~\ref{sec:multidataset}); (iii) surrender curves---per-feature interaction tracked online during training---whose five-seed, four-dataset replication shows that pre-regularization magnitude is only a weak, unstable predictor of which interactions a regularized model retains (Section~\ref{sec:surrender}); (iv) the observation that one-dimensional response sections only become well-defined shape functions once measured interaction is driven down (Figure~\ref{fig:responses}); (v) crystallization, an optional post-processing that converts the penalized model into an exactly additive one at no accuracy cost relative to a from-scratch NAM, and whose comparison against a weight-space alternative of our design (winner-take-all decay) shows that constraining behavior before structure dominates (Section~\ref{sec:exact}); (vi) calibration against classical backfitting with neural smoothers, which quietly matches modern per-feature architectures (Section~\ref{sec:backfit}); (vii) a shared-parameter section model on which sequential backfitting diverges for identifiable structural reasons but joint training of the same object---the shared-section model---is stable and yields our best exactly additive neural model (Section~\ref{sec:sections-model}); and (viii) a six-dataset benchmark including TEAM and the soft penalized model itself, showing that route rankings are regime-dependent and that the near-additive model, at the heuristic default $\lambda=1$, leads the additive family on three datasets---plus three open clinical cohorts on which every additive route beats the dense network, the regime where additivity pays outright (Sections~\ref{sec:multidataset}--\ref{sec:clinical}).

Figure~\ref{fig:taxonomy} organizes the routes by where the constraint acts---structure, behavior, fitting algorithm, or read-out---and marks the paper's contributions within that ontology.

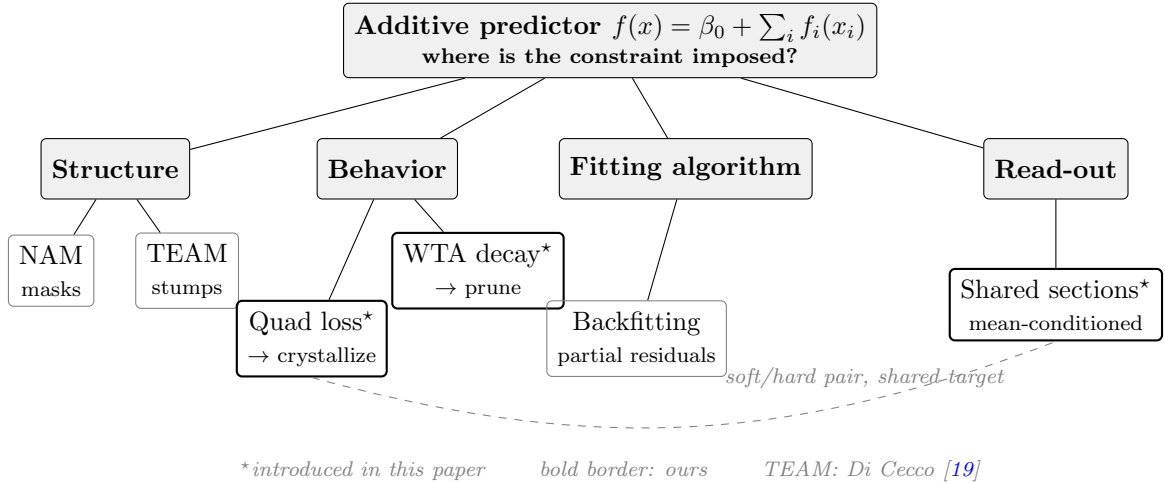
\begin{figure}[t]
\centering
\begin{tikzpicture}[
  every node/.style={font=\small},
  box/.style={draw, rounded corners=2pt, align=center, inner sep=4pt, minimum height=8mm},
  cat/.style={box, fill=gray!12, font=\small\bfseries},
  ours/.style={box, thick, draw=black},
  base/.style={box, draw=gray},
  lab/.style={font=\scriptsize\itshape, text=gray}]
\node[cat] (root) at (0,0) {Additive predictor $f(x)=\beta_0+\sum_i f_i(x_i)$\\[-1pt]{\scriptsize where is the constraint imposed?}};
\node[cat] (S) at (-6.55,-1.7) {Structure};
\node[cat] (B) at (-2.95,-1.7) {Behavior};
\node[cat] (A) at (1.0,-1.7) {Fitting algorithm};
\node[cat] (R) at (5.9,-1.7) {Read-out};
\draw (root) -- (S); \draw (root) -- (B); \draw (root) -- (A); \draw (root) -- (R);
\node[base] (nam) at (-7.4,-3.05) {NAM\\{\scriptsize masks}};
\node[base] (team) at (-5.6,-3.05) {TEAM\\{\scriptsize stumps}};
\node[ours] (quad) at (-3.95,-3.95) {Quad loss$^{\star}$\\{\scriptsize $\to$ crystallize}};
\node[ours] (wta) at (-1.75,-3.05) {WTA decay$^{\star}$\\{\scriptsize $\to$ prune}};
\node[base] (bf) at (0.35,-3.95) {Backfitting\\{\scriptsize partial residuals}};
\node[ours] (sec) at (5.9,-3.5) {Shared sections$^{\star}$\\{\scriptsize mean-conditioned}};
\draw (S) -- (nam); \draw (S) -- (team);
\draw (B) -- (quad); \draw (B) -- (wta);
\draw (A) -- (bf);
\draw (R) -- (sec);
\draw[dashed, gray] (quad.south) to[bend right=18] node[pos=0.75, above=1pt, font=\scriptsize\itshape, text=gray]{soft/hard pair, shared target} (sec.south);
\node[lab] at (0,-5.7) {$^{\star}$introduced in this paper \quad\quad bold border: ours \quad\quad TEAM: \citet{dicecco2024team}};
\end{tikzpicture}
\caption{\textbf{Ontology of the routes compared}, organized by where the additivity constraint acts: on the \emph{structure} (the architecture cannot express interactions), on the \emph{behavior} (a dense model is penalized, then optionally crystallized), on the \emph{fitting algorithm} (components are estimated additively), or on the \emph{read-out} (an unconstrained network is queried through its one-feature sections). The quadrilateral loss additionally serves as the measuring instrument for all four families, and the two routes we introduce are the soft and hard treatments of one constraint, sharing a target but not a parametrization (dashed; Section~\ref{sec:sections-model}).}
\label{fig:taxonomy}
\end{figure}

We emphasize at the outset what this paper does not claim. All experiments use a single tabular dataset and small multilayer perceptrons; we make no claim of generality across domains or scales, and we position this work as a carefully analyzed case study of a simple idea rather than a benchmark sweep. Related penalties on Hessian off-diagonals and functional decompositions exist in the literature (Section~\ref{sec:landscape}); training-time penalization of explanations also predates us \citep{ross2017right,rieger2020cdep}; the quadrilateral loss itself---a finite-difference penalty targeting global additivity with an exact characterization---is, to our knowledge, new, as are the $\lambda$-interpolation analysis and the surrender-curve decomposition built on it.

\section{The Additive Modeling Landscape}
\label{sec:landscape}

An additive model predicts $f(x)=\beta_0+\sum_i f_i(x_i)$ with arbitrary univariate components: each $f_i$ is a plottable, auditable object, and the price is the exclusion of interactions. The classical instantiation fits the components with smoothers via backfitting \citep{hastie1986gam,buja1989linear}. The tree lineage replaces smoothers with depth-one ensembles---boosted stumps are additive by construction \citep{friedman2001greedy}---running through the intelligible models of \citet{lou2012intelligible}, GA$^2$M's screened pairwise terms \citep{lou2013ga2m}, and the Explainable Boosting Machine \citep{nori2019interpretml}; TEAM \citep{dicecco2024team,dicecco2025xai_ch6}, described in Section~\ref{sec:setup}, is a compact member, and on tabular data this lineage remains hard to beat \citep{grinsztajn2022trees}. Its canonical deployment is clinical risk: \citet{caruana2015intelligible} famously used an intelligible additive model on pneumonia data to surface the rule \emph{asthma implies lower mortality risk}---a triage artifact that a dense model would have silently exploited---and that episode is the template for what our surrender curves are built to do online: expose which interactions a model is retaining, so that a domain expert can decide whether they are signal or artifact. The neural lineage gives each feature its own network \citep{agarwal2021nam}, with successors sharing bases across features \citep{radenovic2022nbm}, modeling interactions via low-rank polynomial decompositions \citep{dubey2022spam}, building components from differentiable trees \citep{chang2022node}, adding structured pairwise terms \citep{yang2021gaminet}, or treating the sub-networks in a Bayesian fashion, with credible intervals and interaction ranking \citep{bouchiat2024lanam}; our shared-section model (Section~\ref{sec:sections-model}) is the limiting case of basis sharing, with the whole network as basis and no per-feature parameters. Adjacent threads include Kolmogorov--Arnold Networks \citep{liu2024kan}, whose one-layer form with identity outer functions is exactly a GAM, and concurvity regularization \citep{siems2023concurvity}, which penalizes correlation \emph{between} fitted components to make decompositions identifiable under dependent features---complementary to our penalty on interactions \emph{within} a dense model.

On the measurement side, our statistic descends from the functional-ANOVA tradition \citep{hooker2007anova,sobol2001}, the partial-dependence-based $H$-statistic of \citet{friedman2008rules}, interaction detection by restricted retraining \citep{sorokina2008groves}, and post-hoc interaction attribution \citep{lundberg2017shap,janizek2021ih,tsang2018nid}; \citet{tsang2018neural} constrain interactions through architecture discovery, and training-time penalization of explanations predates us in the form of gradient penalties \citep{ross2017right}, attribution priors \citep{erion2021attribution}, and contextual decomposition penalization of practitioner-specified attributions \citep{rieger2020cdep}. Pointwise second-derivative penalties are ill-posed for piecewise-linear networks, which motivates our finite-difference formulation; the quadrilateral loss itself---a nonlocal penalty targeting global additivity with an exact characterization---is, to our knowledge, new, as are the $\lambda$-interpolation analysis and the surrender-curve decomposition built on it. Our crystallization step borrows saliency from magnitude--gradient pruning \citep{molchanov2017pruning} and group sparsity \citep{wen2016ssl}.

\section{The Quadrilateral Loss}
\label{sec:method}

We introduce the quadrilateral loss, a differentiable training penalty whose population version exactly characterizes additivity and which remains informative for piecewise-linear networks. Adjacent constructions exist---mixed differences as post-hoc diagnostics in the functional-ANOVA tradition \citep{hooker2007anova}, and training-time penalization of local, practitioner-specified attributions \citep{ross2017right,rieger2020cdep,erion2021attribution}---but to our knowledge the statistic below has not previously been used as a training objective, nor tied to a global additivity guarantee.

Let $f_\theta:\R^d\to\R$ be a network and let $x,\tilde{x}$ be two inputs. For a coordinate $i$, define the single-coordinate swaps $x^{(i)} = (\tilde{x}_i, x_{-i})$ and $\tilde{x}^{(i)} = (x_i, \tilde{x}_{-i})$, where $x_{-i}$ denotes all coordinates except $i$. The \emph{quadrilateral residual} is
\begin{equation}
\label{eq:quad}
Q_i(x,\tilde{x}) \;=\; f_\theta(x) \;-\; f_\theta(x^{(i)}) \;-\; f_\theta(\tilde{x}^{(i)}) \;+\; f_\theta(\tilde{x}).
\end{equation}
This is the second-order mixed difference of $f_\theta$ along the coordinate-$i$ direction versus all remaining directions, evaluated at the four corners of an axis-aligned quadrilateral. If $f_\theta$ is additive, every term $f_j(\cdot)$ for $j\ne i$ appears once with each sign and cancels, as does $f_i(\cdot)$, so $Q_i \equiv 0$. Conversely:

\begin{proposition}
\label{prop:char}
Let $\mathcal{X} = \mathcal{X}_1\times\cdots\times\mathcal{X}_d$ be a product domain and $f:\mathcal{X}\to\R$. Then $Q_i(x,\tilde{x})=0$ for all $x,\tilde{x}\in\mathcal{X}$ and all $i$ if and only if $f(x) = \sum_i f_i(x_i)$ for some functions $f_i:\mathcal{X}_i\to\R$.
\end{proposition}

\begin{proof}
Sufficiency is the cancellation above. For necessity, fix a reference point $c\in\mathcal{X}$ and define $f_i(x_i) = f(x_i, c_{-i}) - \tfrac{d-1}{d}f(c)$. Setting $\tilde{x}=c$ in $Q_i(x,c)=0$ gives $f(x) = f(x^{(i)}\!\!\mid_{\tilde x = c}) + f(x_i,c_{-i}) - f(c)$, i.e., replacing coordinate $i$ of any point by $c_i$ changes $f$ by an amount depending on $x_i$ alone. Applying this coordinate by coordinate telescopes $f(x)$ down to $f(c)$ plus a sum of single-coordinate increments, which is the claimed additive form.
\end{proof}

Two properties matter in practice. First, \eqref{eq:quad} is \emph{kink-robust}: it integrates interaction across activation boundaries rather than sampling the gradient at a point, so it detects interactions in ReLU networks for which $\partial^2 f/\partial x_i \partial x_j = 0$ almost everywhere. Second, it is cheap: four forward passes, no double backpropagation, and the four corner evaluations batch trivially.

\paragraph{Training objective.} We penalize the empirical second moment of the residual over data pairs and coordinates,
\begin{equation}
\label{eq:loss}
\mathcal{L}(\theta) \;=\; \underbrace{\E\big[(f_\theta(x)-y)^2\big]}_{\text{task}} \;+\; \lambda\,\frac{1}{|S|}\sum_{i\in S}\;\E_{x,\tilde{x}}\big[Q_i(x,\tilde{x})^2\big],
\end{equation}
where $\tilde{x}$ is obtained by permuting the minibatch and $S$ is a random subset of coordinates resampled each step ($|S|=4$ throughout). Sampling $\tilde x$ from the data distribution rather than a product measure means the penalty is enforced \emph{on-manifold}: the network may retain interactions in regions unsupported by data. For interpretability applications this is usually the desired semantics; for worst-case guarantees it is not, and we return to this limitation in Section~\ref{sec:limits}.

\paragraph{Connection to the Shapley-GAM.} The statistic has an exact game-theoretic meaning. \citet{bordt2023shapley} show that every subset-compliant value function induces a unique functional decomposition $f=\sum_{S\subseteq[d]}f_S(x_S)$ via the M\"obius transform---the \emph{Shapley-GAM}---and that for the interventional SHAP value function $v(x,S)=\E_{z}[f(z)\,|\,do(x_S)]$ the pairwise component is precisely a quadrilateral: $f_{ij}=\E[f|do(x_i,x_j)]-\E[f|do(x_i)]-\E[f|do(x_j)]+\E[f]$. Our statistic is the two-point Monte Carlo version of this construction, and the correspondence is exact: writing $v(x,S)=\sum_{L\subseteq S}f_L$ and taking the expectation of \eqref{eq:quad} over the partner point $\tilde{x}\sim\mathcal{D}$ with $x$ fixed gives $\E_{\tilde{x}}[Q_i(x,\tilde{x})] = f(x)-v(x,[d]\!\setminus\!i)-v(x,\{i\})+v(x,\emptyset) = \sum_{L\ni i,\,|L|\ge2} f_L(x_L)$, the sum of \emph{all} interaction components of the interventional Shapley-GAM that contain coordinate $i$. The quadrilateral loss therefore penalizes the $\ell^2$ mass of the Shapley-GAM's interaction components, coordinate by coordinate, and driving it to zero is equivalent to the Shapley-GAM being of order one. Two consequences follow (a third, structural one---that additive functions are the fixed points of a mean-anchored section read-out, making our penalty the soft counterpart of a hard projection---is developed with the section model in Section~\ref{sec:sections-model}). First, by the recovery theorem of \citet{bordt2023shapley}, once a model is (near-)additive its ordinary interventional SHAP values coincide with its shape functions \emph{even under arbitrarily dependent features}---so $\lambda$ is also a dial for the faithfulness of standard SHAP explanations, quantified by the same statistic. Second, the surrender curves of Section~\ref{sec:surrender} become a tractable window onto an otherwise exponential object: the per-coordinate interaction mass of the Shapley-GAM, which \citet{bordt2023shapley} estimate by enumerating all $2^d$ subsets, is probed here at $O(d)$ forward passes per checkpoint, online, during training. 

\paragraph{Higher-order probes and uncertainty.} The construction generalizes beyond single coordinates: swapping a subset $S$ of coordinates between $x$ and $\tilde{x}$ and taking the alternating sum over the $2^{|S|}$ corners yields a one-sample Monte Carlo estimate of $\sum_{L\supseteq S} f_L(x)$, the joint interaction mass of all Shapley-GAM components containing $S$---for pairs, an $8$-point statistic that screens candidate interaction terms for a GA$^2$M at a handful of forward passes per pair, in place of exponential subset enumeration. Bayesian treatments of NAMs also rank candidate interaction pairs \citep{bouchiat2024lanam}; our methods acquire uncertainty by composition rather than redesign---the penalty is a loss term through which any Laplace or ensemble treatment of the dense network trains unchanged, the section read-out inherits credible bands by evaluating the posterior on the cross, and TEAM's bagged replicates already form a bootstrap ensemble whose between-bag spread is a confidence band at no extra cost. We use only the single-coordinate probe in this paper.

\paragraph{The knob, honestly stated.} As $\lambda$ ranges over $[0,\infty)$, \eqref{eq:loss} interpolates between an unconstrained dense network and an approximately additive one. We stress that intermediate $\lambda$ does \emph{not} correspond to a structured model class such as a GA$^2$M of fractional order: the intermediate models are dense networks with uniformly small but nonzero interactions of all orders. The interpolation is a trade-off curve in behavior space, not a nested family of hypotheses. Where a structured intermediate class is required, the penalty extends naturally by restricting the sum in \eqref{eq:loss} to coordinate swaps \emph{across} prescribed feature blocks, leaving intra-block interactions free; we leave a systematic study of this block variant to future work.

\section{Experimental Setup}
\label{sec:setup}

All experiments use the California Housing dataset \citep{pace1997calhousing} (20{,}433 block groups after dropping rows with missing bedroom counts; eight standardized features; target median house value in units of \$100{,}000), with a fixed random 80/20 train/test split. Unless stated otherwise the architecture is a two-hidden-layer MLP ($64+64$ units, GELU) trained with Adam \citep{kingma2015adam} at learning rate $10^{-3}$, batch size 512, 120 epochs, with $\lambda$ annealed linearly from zero after a 15-epoch warmup; the $\lambda$-sweep of Section~\ref{sec:sweep} uses a single-hidden-layer network ($H=64$, 80 epochs, 3 seeds per point) to keep the sweep cheap, and all comparisons in that section are against the \emph{same} architecture at $\lambda=0$. We measure interactions with the test-set statistic $\overline{Q^2} = \frac{1}{d}\sum_i \E[Q_i^2]$ estimated by batch permutation (3 replicates); for reference, the test target variance is $1.31$.

\paragraph{Routes compared.} For self-containedness we describe every method at the level needed to reimplement it.

\emph{NAM from scratch.} The same MLP with block-diagonal binary masks fixed at initialization: first-layer units are pre-assigned to features in balanced groups (unit $j$ may connect only to feature $j \bmod d$), second-layer units connect only to first-layer units of their own group, and the linear output layer sums all groups. The masks are re-applied after every optimizer step, so each hidden path processes exactly one feature and the model is additive by construction \citep{agarwal2021nam}. We use standard GELU units rather than the exp-centered (ExU) activations of the original NAM; we return to this choice, which is not innocent, in Section~\ref{sec:sections}.

\emph{WTA decay (ours).} Dense training in which, after a 15-epoch warmup, every non-maximal input weight of each first-layer unit is multiplied by $(1-\varepsilon)$ at each step ($\varepsilon$ annealed to $0.02$), so each unit's receptive field contracts toward a single self-selected feature; a grouped analogue acts at the second layer with a $2.5\times$ stronger rate. Training is followed by hard pruning to the winner topology (with the capacity-balancing repair of Section~\ref{sec:exact}) and fine-tuning under the masks.

\emph{Quadrilateral penalty (ours).} Dense training with \eqref{eq:loss}, $\lambda$ annealed to 20, followed by the same pruning-and-fine-tuning protocol; first-layer feature assignments are extracted from an exponential moving average of the weight--gradient saliency $|w\odot\nabla_w \mathcal{L}|$ \citep{molchanov2017pruning}.

\emph{Backfitting.} Backfitting \citep{hastie1986gam} uses single-feature networks ($1{\to}64{\to}1$, GELU) as smoothers, fitting them cyclically on partial residuals $y-\alpha-\sum_{j\ne i}f_j(x_j)$, eight inner epochs per component per sweep, warm-started across 40 sweeps, with each component mean-centered after its update and the intercept $\alpha$ absorbing the shifts.

\emph{Shared-section model (ours).} A single ordinary dense network $g_\theta$ read out through Eq.~\ref{eq:secmodel}: the additive predictor is the sum of $g_\theta$'s mean-conditioned one-feature sections, trained end-to-end by backpropagating the task loss through all $d{+}1$ masked forward passes jointly.

\emph{TEAM.} A Tree Ensemble Additive Model \citep{dicecco2024team,dicecco2025xai_ch6}: a bagged ensemble ($B$ bootstrap replicates) of gradient-boosted regression trees of depth one. Additivity is exact and comes from the weak learner, not from any constraint machinery: a depth-one tree (stump) is a step function of a single feature, so a boosted sum of stumps is a sum of per-feature step functions, $\sum_m \nu\, h_m(x_{i_m}) = \sum_i f_i(x_i)$ with piecewise-constant $f_i$; bagging averages additive models and is therefore additive. Boosting chooses at every stage which feature to split, so capacity is allocated across features adaptively by the functional gradient rather than fixed per branch; bagging reduces the variance of that greedy allocation. This is also where TEAM departs from the EBM lineage of \citet{lou2012intelligible,nori2019interpretml}: EBMs fit shape functions by \emph{cyclic} gradient boosting, visiting every feature in round-robin with a small learning rate to wash out order effects, which costs a full pass over the feature set per boosting round; TEAM lets each stage select its feature greedily by standard split search and delegates the de-biasing of that greed to bagging, whose replicates are embarrassingly parallel. The result is a simpler and faster training loop at the price of coarser per-feature budget control. Shape functions are read out by the same mean-conditioning used in Eq.~\ref{eq:secmodel}. We report TEAM at $B{=}5$ bags of 1{,}000 stumps (learning rate $0.1$).

\begin{figure}[t]
\centering
\includegraphics[width=\linewidth]{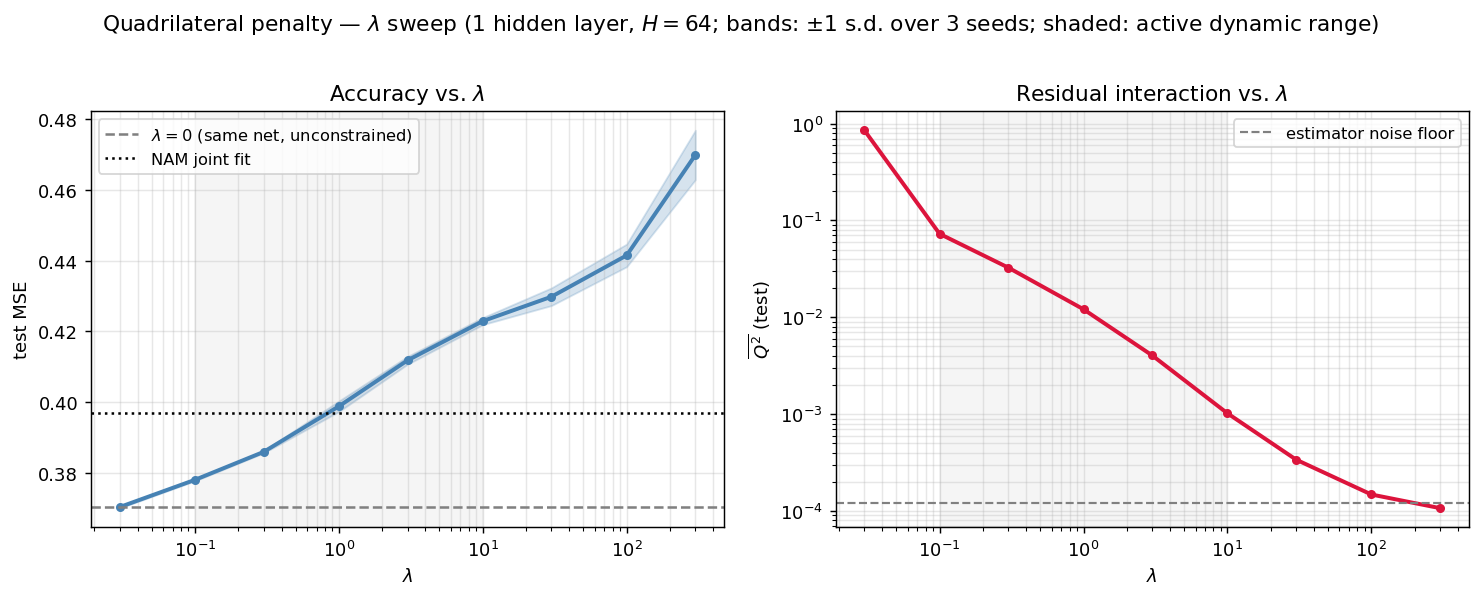}
\caption{\textbf{Accuracy--interaction trade-off under the quadrilateral penalty} (single hidden layer, $H=64$; mean $\pm$ s.d.\ over 3 seeds). \emph{Left}: test MSE versus $\lambda$; the dashed line is the identical architecture at $\lambda=0$ and the dotted line a NAM of matched width trained from scratch. \emph{Right}: residual interaction $\overline{Q^2}$ versus $\lambda$ (log--log). Interaction falls by nearly four orders of magnitude, from $0.86$ to $1.1\times10^{-4}$, while MSE rises from $0.370$ to $0.470$; at $\lambda=0.1$, $91\%$ of the interaction mass is removed for $+0.008$ MSE. The right curve saturates near $10^{-4}$, the noise floor of the stochastic estimator rather than exact additivity.}
\label{fig:sweep}
\end{figure}

\paragraph{Reproducibility.} Code reproducing every experiment, table, and figure---the loss, all training routes, the balance heuristic, the Hyperband protocol, and the surrender-curve instrumentation---is available at \texttt{github.com/AntonioDiCecco/quadrilateral-loss}.

\section{What the Loss Reveals: a Dial and an Observable}
\label{sec:reveals}
The single coefficient $\lambda$ turns additivity into a training-time dial; the coordinate-wise decomposition of the statistic turns interaction into an online observable; and driving the statistic down determines when one-dimensional plots earn an interpretation. This section develops the three in turn.

\subsection{The Trade-off is Strongly Asymmetric}
\label{sec:sweep}

Figure~\ref{fig:sweep} reports the sweep $\lambda\in\{0, 0.1, 0.3, 1, 3, 10, 30, 100, 300\}$. Three observations.

\emph{Most interactions are cheap to remove.} At $\lambda=0.1$ the interaction statistic drops from $0.855$ to $0.073$---a $91\%$ reduction---while test MSE moves from $0.3703$ to $0.3780$. By $\lambda=10$ interactions are down three orders of magnitude ($1.0\times10^{-3}$) at MSE $0.423$, still short of the full additive cost. The dense optimum evidently encodes its function with substantial interaction redundancy: alternative, nearly-additive parameterizations of almost the same function exist nearby in function space, and a weak penalty suffices to select them. The economically meaningful interactions---those the network genuinely needs---are concentrated in the last order of magnitude of suppression, which is where the accuracy price is paid.

\emph{The useful dynamic range of $\lambda$ is narrow.} Both curves are steep only for $\lambda\in[0.1,10]$; below, the penalty is inert, and above, the interaction estimate saturates at ${\sim}10^{-4}$. This saturation is a property of the estimator, not the function: with finite batches and $|S|=4$ sampled coordinates per step, the stochastic penalty cannot certify exact additivity, only drive its estimate to the noise floor. Exactness requires the crystallization step of Section~\ref{sec:exact}.

\emph{Near-additive dense networks can match structural NAMs.} At $\lambda=1$ the single-layer penalized network attains MSE $0.399$ with residual interaction $1.2\times10^{-2}$---statistically indistinguishable from the from-scratch NAM (${\approx}0.39$, Table~\ref{tab:multi}). The ``almost additive'' regime is thus a legitimate operating point rather than a no-man's-land, and there is a principled reason to prefer it that goes beyond accuracy. Exact additivity pushes complexity \emph{into} the components: the best additive projection of a genuinely interacting target can demand shape functions of extreme roughness---the observation that led the NAM authors to engineer ExU activations for ``jumpy'' univariate fits \citep{agarwal2021nam}---and a shape function jagged enough to absorb what interactions used to express is formally one-dimensional but no longer humanly readable, which defeats the purpose of the constraint. This is the expressivity paradox of additive models: perfect additivity and interpretable components can be in tension. The paradox has a classical pedigree. The Kolmogorov--Arnold superposition theorem \citep{kolmogorov1957representation} shows that \emph{every} continuous multivariate function is a finite superposition of univariate functions---so one-dimensionality of the building blocks, by itself, carries no interpretability guarantee whatsoever---and the univariate functions that make the theorem true are necessarily wild: they cannot in general be taken smooth \citep{vitushkin1954hilbert}, and the standard constructions produce continuous but fractal-like inner functions. The theorem is thus the extreme statement of the trade our sweep exhibits in miniature: multivariate complexity does not disappear when arity is constrained, it is displaced into the regularity of the components. A GAM is far weaker than a full superposition---one layer, no outer function---so it cannot absorb everything; but its best additive projection of an interacting target already sits on the same slope, and the roughness that ExU activations were engineered to reach is that displaced complexity surfacing. Soft enforcement dissolves the dilemma from the side: at moderate $\lambda$ the model retains a small interaction budget precisely where forcing additivity would deform the components most, and the resulting sections stay smooth while the quadrilateral statistic certifies, quantitatively, how far from additive the model still is. Near-additivity with a measured residual can be more interpretable than exact additivity with wild components.

\paragraph{A regularization sweet spot on small data.} The California Housing sweep is monotone: every unit of additivity costs accuracy. But Section~\ref{sec:multidataset} shows that on smaller datasets the additive \emph{class} outperforms the dense network---which predicts that there the trade-off curve should dip \emph{below} the unconstrained baseline before rising. It does, with a strength that itself depends on the dataset (Figure~\ref{fig:sweetspot}; ten seeds, paired per seed since penalty runs share split and initialization with their baseline). On Wine the effect is unambiguous: test MSE falls from $0.480$ at $\lambda=0$ to $0.454$ at $\lambda=3$, the paired improvement holds on \emph{ten of ten} seeds for every $\lambda\in[0.3,3]$ ($-0.023\pm0.009$ at $\lambda=1$), the interaction statistic drops an order of magnitude, and even $\lambda=30$ remains at or below the baseline on seven of ten seeds. On Concrete a ten-seed replication demotes our initial three-seed reading: the dip survives only at the smallest penalty ($\lambda=0.3$, $-0.8\pm1.4$ paired, eight of ten seeds) and $\lambda\ge3$ already costs accuracy---the earlier, larger dip was partly seed luck. The honest summary is therefore conditional rather than universal: in the small-$n$ regime there \emph{can} exist $\lambda$ at which the model is simultaneously more accurate and more additive than the unconstrained one, decisively so on Wine, marginally on Concrete, and where it happens the quadrilateral penalty is not a constraint purchased with accuracy but a regularizer paid for by nothing---the interactions were the overfitting.

\begin{figure}[t]
\centering
\includegraphics[width=.9\linewidth]{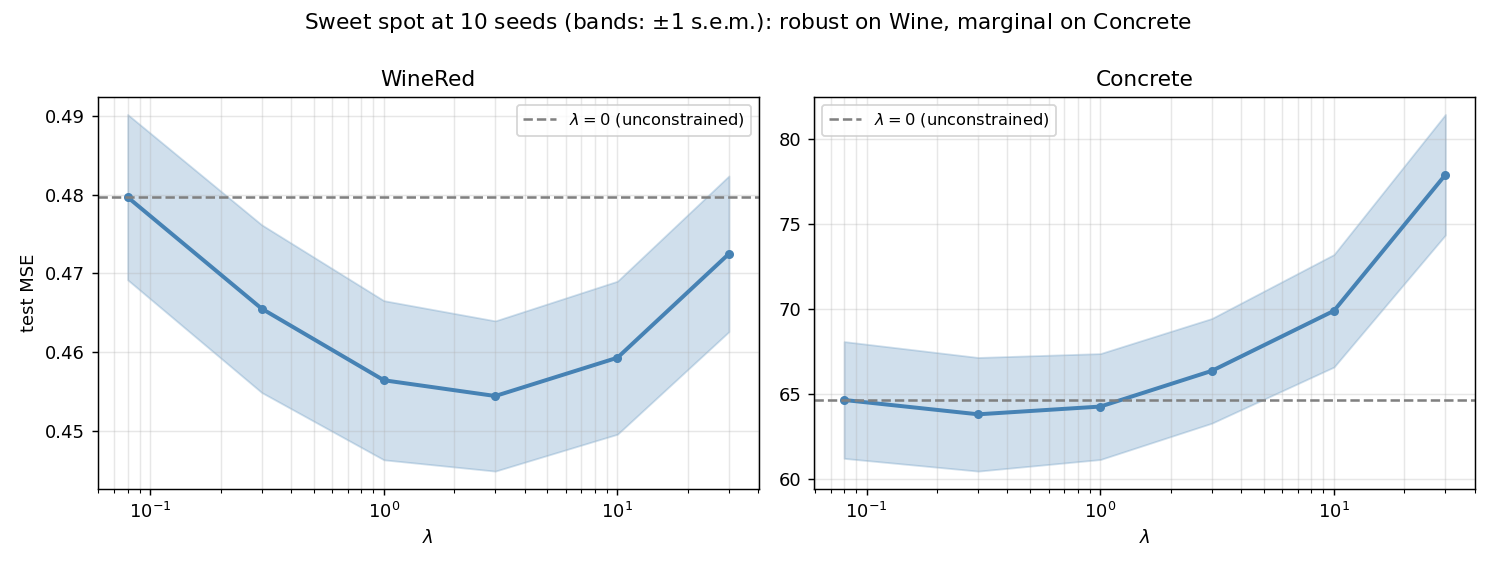}
\caption{\textbf{Win--win regime of the quadrilateral penalty} on the two datasets where the additive class beats the dense network (2-hidden-layer net, 10 seeds, bands $\pm$1 s.e.m.). On Wine the dip below the unconstrained baseline (dashed) is seed-universal for $\lambda\in[0.3,3]$; on Concrete it is marginal and confined to the smallest penalty. Where the dip exists, the penalty regularizes rather than trades.}
\label{fig:sweetspot}
\end{figure}

\subsection{Surrender Curves: Interaction as a Training-Time Observable}
\label{sec:surrender}

Because the quadrilateral statistic decomposes by coordinate, the penalty doubles as a free online diagnostic: tracking $\E[Q_i^2]$ per feature during training yields \emph{surrender curves} showing which features' interactions resist the regularizer and which capitulate (CalHousing panel of Figure~\ref{fig:surrenderall}). This is information no structurally constrained model can emit---a NAM never had interactions to surrender---and it turns post-hoc interaction detection \citep{sobol2001,lundberg2017shap} into a live training signal.

\begin{table}[h]
\centering
\small
\begin{tabular}{lcccc}
\toprule
Feature & $\E[Q_i^2]$ at warmup end & $\E[Q_i^2]$ at convergence & Reduction & Rank at convergence \\
\midrule
MedInc    & $0.090$ & $8.0\times10^{-4}$ & $113\times$  & 2 \\
HouseAge  & $0.073$ & $4.8\times10^{-4}$ & $153\times$  & 7 \\
AveRooms  & $0.592$ & $8.9\times10^{-4}$ & $667\times$  & 1$^{*}$ \\
AveBedrms & $0.662$ & $6.6\times10^{-4}$ & $1006\times$ & 4 \\
Pop       & $0.022$ & $3.2\times10^{-4}$ & $69\times$   & 8 \\
AveOccup  & $0.011$ & $1.9\times10^{-3}$ & $6\times$    & 1 \\
Lat       & $0.899$ & $6.7\times10^{-4}$ & $1342\times$ & 3 \\
Lon       & $0.870$ & $5.7\times10^{-4}$ & $1518\times$ & 5 \\
\midrule
Total & $3.22$ & $6.3\times10^{-3}$ & $511\times$ & \\
\bottomrule
\end{tabular}
\caption{Per-feature interaction $\E[Q_i^2]$ on the test set at the end of the unconstrained warmup (epoch 15, the peak of freely grown interaction) and at convergence (mean of the last five checkpoints), with the reduction factor achieved by the penalty. The largest interactions (Lat, Lon, the ratio features) are suppressed by three orders of magnitude, while AveOccup---smallest at warmup---is reduced only $6\times$ and ends as the largest residual ($^{*}$AveRooms ranks first only through its transient spikes; on the smoothed tail AveOccup leads). On this run (seed 0) warmup magnitude and final residual are uncorrelated ($\rho_s=-0.05$); Figure~\ref{fig:surrenderseeds} shows the seed-replicated picture.}
\label{tab:surrender}
\end{table}

The curves corrected our own prior. We expected latitude and longitude---the carriers of the dataset's dominant interaction---to be the last holdouts; instead they surrender to ${\sim}7\times10^{-4}$, while the stubborn features are average occupancy (${\sim}2\times10^{-3}$, drifting upward late) and median income. Two readings, not mutually exclusive. First, $\E[Q_i^2]$ is a variance-weighted measure evaluated on random data pairs: the heavy-tailed ratio features (rooms, bedrooms, occupancy) produce extreme swap pairs whose quadrilateral residuals dominate the average, which also explains their outsized pre-penalty peaks. Second, the geographic interaction, while large, is apparently \emph{easy to trade away}---the network finds nearby near-additive substitutes for it---whereas the occupancy--income interactions are entangled with the task gradient in a way that keeps regenerating them, visible as the late upward drift and the transient re-emergence spikes around epochs 59 and 103. Table~\ref{tab:surrender} quantifies the effect: reduction factors range from $1518\times$ (longitude) down to $6\times$ (occupancy). We caution against reading the strong positive correlation between warmup magnitude and reduction factor as a finding: since the factor is the ratio peak/final, it correlates with the peak mechanically whenever the final residual is independent of the peak---and independence is precisely what we observe (below). Either way, the operational point stands: which interactions are expensive to remove is an empirical property of the optimization, not something readable off post-hoc importance scores, and the surrender curves measure it directly during training.

\paragraph{A dynamical-systems analogy.} We note, with the caution due to any cross-field metaphor, that the structure of the finding echoes KAM theory \citep{arnold1963kam}. There, an integrable Hamiltonian---the mechanical analogue of a separable, interaction-free system---is perturbed, and the question is which invariant tori survive; the celebrated answer is that survival is governed not by any notion of a torus's size but by an arithmetic property of its frequencies (a Diophantine non-resonance condition), with resonant tori destroyed at arbitrarily small perturbation. Our setting runs the perturbation in the opposite direction---the penalty pushes an interacting system \emph{toward} separability---but exhibits the same signature: which interactions survive the perturbation is decided not by their magnitude but by a structural property of how they are entangled with the task gradient, a property the magnitude ranking barely predicts. We do not claim a formal correspondence; we record the analogy because it correctly reframes the question the surrender curves answer---from ``how big is this interaction?'' to ``how resonant is it with what the loss needs?''---and because it suggests that a quantitative stubbornness criterion, playing the role of the Diophantine condition, may be identifiable.

\paragraph{Replication across datasets and seeds: magnitude is a weak, unstable predictor of stubbornness.} Repeating the experiment on Wine, Abalone, and Boston (same protocol, per-dataset batch sizes as in Section~\ref{sec:multidataset}) and---crucially---across five seeds per dataset revised our own first reading, twice. On the initial seed, the rank correlation between a feature's warmup interaction magnitude and its final residual appeared to vanish on all four datasets ($|\rho_s|\le0.12$), suggesting clean decoupling. Five seeds show that this was a fortunate draw: per-seed correlations range from $-0.26$ to $+0.81$, with per-dataset means of $+0.37\pm0.31$ (CalHousing), $+0.20\pm0.26$ (Wine), $+0.39\pm0.27$ (Abalone), and $+0.05\pm0.14$ (Boston); pooling all twenty runs, the grand mean is $\rho_s=+0.25\pm0.29$, significantly positive (Wilcoxon $p=0.002$) but weak (Figure~\ref{fig:surrenderseeds}). The corrected claim is therefore not decorrelation but \emph{unreliability}: initial interaction magnitude carries some signal about final stubbornness, yet so little---and so seed-dependent---that on any single training run the magnitude ranking routinely inverts (seed-level $\rho_s$ crosses zero on three of four datasets). The identities of the stubborn features remain dataset-specific and plausible (shucked and whole weight on Abalone, whose near-collinear relation to total weight is genuinely interactive; rooms and \%lower-status on Boston), consistent with stubbornness being a real feature-level property whose relation to raw magnitude is loose. The practical consequence for anyone using post-hoc interaction scores to decide which pairwise terms a GA$^2$M should keep survives in weakened but usable form: the magnitude ranking explains at best a small fraction of the retention ranking's variance, and only an online measurement like the surrender curves, run on the actual training in question, reveals which interactions that run will keep.

\begin{figure}[t]
\centering
\includegraphics[width=.72\linewidth]{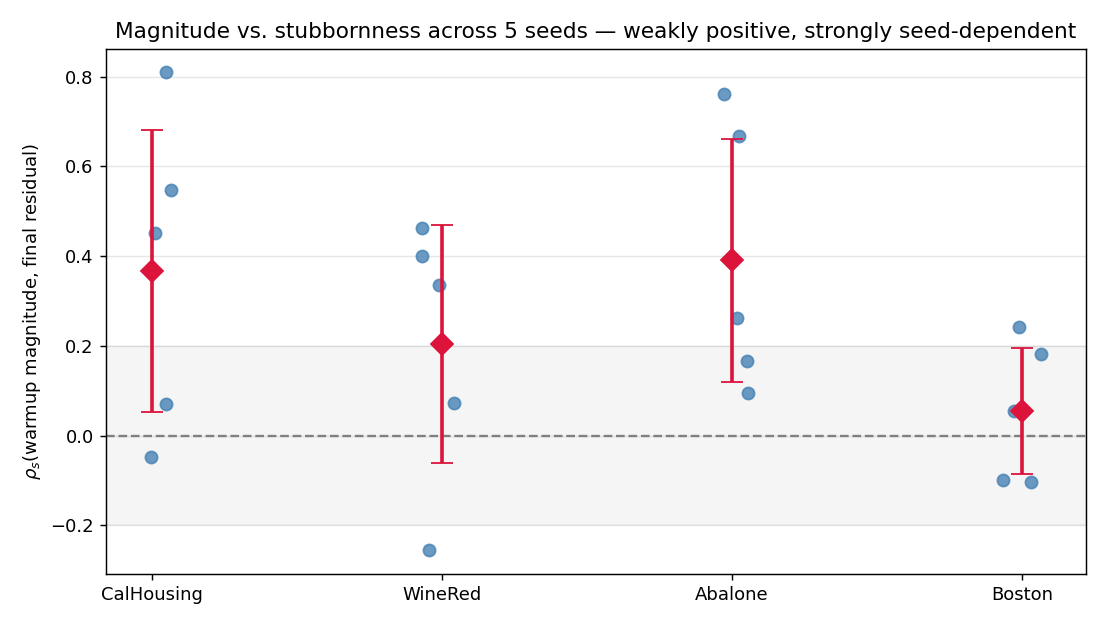}
\caption{\textbf{Seed replication of the magnitude--stubbornness relation.} Each dot is one training run's Spearman correlation between per-feature warmup interaction magnitude and final residual; red diamonds: per-dataset mean $\pm$ s.d.\ over five seeds. The relation is weakly positive on average ($+0.25\pm0.29$ pooled, Wilcoxon $p=0.002$) and crosses zero within seeds on three of four datasets: a single run's ranking is unreliable, which is itself the operational finding.}
\label{fig:surrenderseeds}
\end{figure}

\begin{figure}[t]
\centering
\includegraphics[width=.92\linewidth]{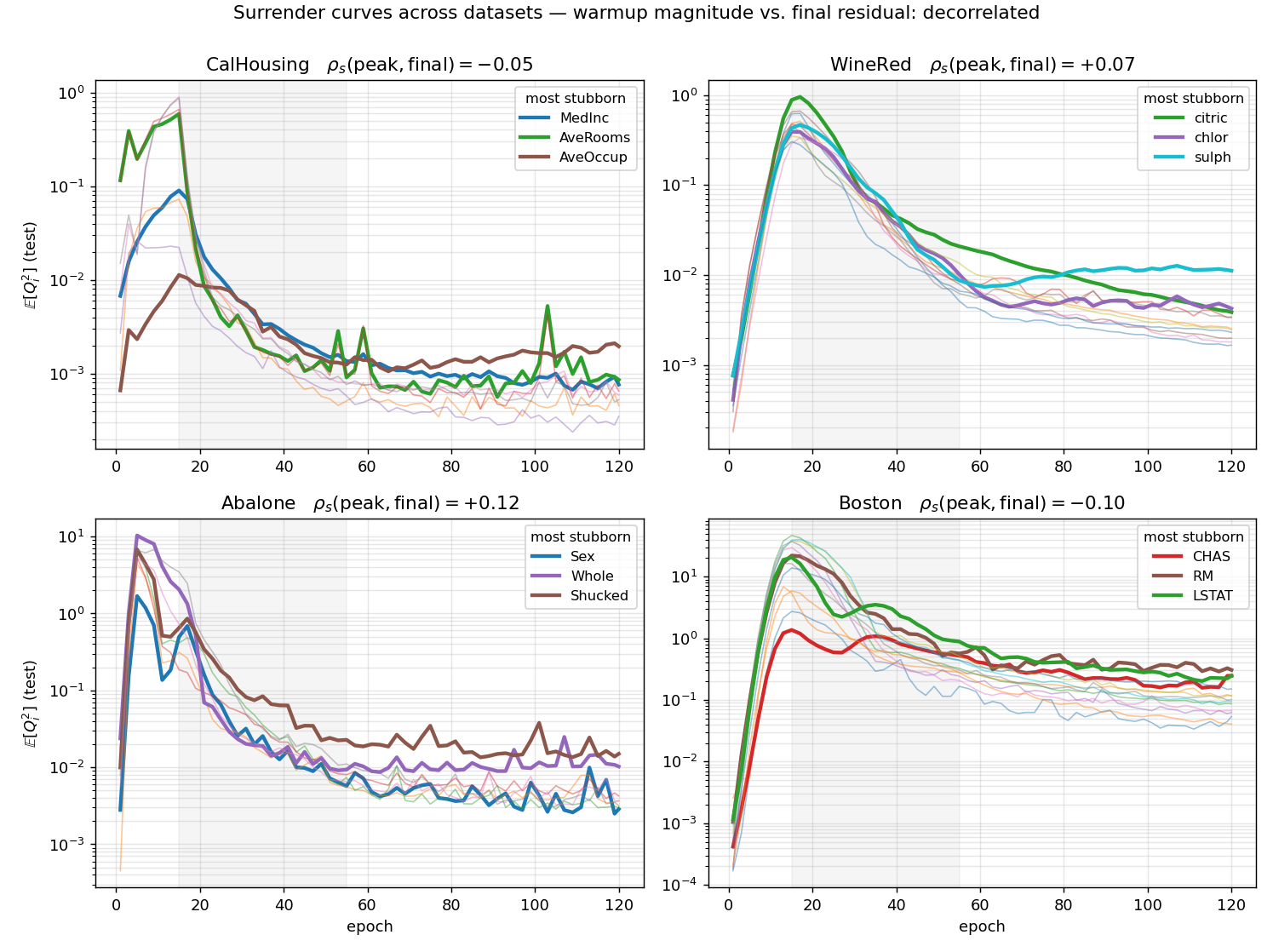}
\caption{\textbf{Surrender curves on four datasets} (three most stubborn features highlighted per panel; shaded band = penalty annealing). Panel titles report $\rho_s$ between warmup magnitude and reduction factor; panel curves are shown for seed 0; the seed-replicated statistics are in Figure~\ref{fig:surrenderseeds}.}
\label{fig:surrenderall}
\end{figure}

\subsection{When Do Sections Become Shape Functions?}
\label{sec:sections}

Figure~\ref{fig:responses} shows, for seven values of $\lambda$, the one-dimensional sections $x_i \mapsto f_\theta(x_i, \bar{x}_{-i})$ with the remaining coordinates at their means. These are the objects practitioners routinely plot as ``the effect of feature $i$.'' The figure makes the epistemic status of such plots visible. For MedInc and HouseAge the sections are essentially invariant across the entire $\lambda$ range, including $\lambda=0$: their effects are additive already in the unconstrained model, and the penalty has nothing to remove. For latitude and longitude---the interaction carriers---the $\lambda=0$ section is the most irregular curve in each panel and its shape drifts as $\lambda$ grows, converging only around $\lambda\gtrsim10$ to a stable profile shared across penalization strengths. The average-ratio features (AveRooms, AveBedrms, AveOccup) show a related effect in their sparse tails, where apparent structure at $\lambda=0$ flattens under the penalty; caution is warranted there regardless, as the standardized tails beyond ${\pm}2\sigma$ contain few data points.

The interpretation is not that the penalty ``smooths'' the sections. It is that at $\lambda=0$ a section is a one-dimensional slice through a strongly interacting function, and its shape depends on the arbitrary choice of conditioning point $\bar{x}_{-i}$---slicing at a different context would give a different curve, which is precisely what interaction means. Only as $\overline{Q^2}\to 0$ does the section become invariant to the conditioning point and thus interpretable as a shape function $f_i$. The $\lambda$-convergence visible in Figure~\ref{fig:responses} is therefore a direct visualization of sections \emph{acquiring} well-definedness. It is the training-time dual of a phenomenon \citet{bordt2023shapley} exhibit post-hoc: their Shapley-value partial-dependence scatter collapses onto the component functions as the \emph{explanation order} $n$ grows toward $d$, ours as the \emph{interaction budget} shrinks toward zero---two paths to the same identified object. A sharper version of this diagnostic---plotting, at each $\lambda$, a bundle of sections conditioned at multiple context points and watching the bundle collapse---would quantify section validity per feature; we propose it as a cheap sanity check for any practitioner reading partial-dependence-style plots off a dense model.

\section{Routes to Exact Additivity}
\label{sec:routes-exact}
Near-additivity suffices for many uses, but exactness is sometimes the point. This section covers how to get it: a brief account of crystallization as post-processing for the penalized model, then the two full modeling routes---classical backfitting with neural smoothers, and reading an additive predictor out of a single shared network, which will turn out to be not a separate design but the hard-projection counterpart of the quadrilateral penalty itself.

\subsection{From Near to Exact: Crystallization as Optional Post-Processing}
\label{sec:exact}

The penalized model is additive in behavior, not in structure: the stochastic estimator bottoms out near $10^{-4}$, so exactness, where required, needs one further step. \emph{Crystallization} converts the penalized dense network into a masked NAM topology---first-layer units assigned to features by the saliency $|w\odot\nabla_w\mathcal{L}|$, second-layer units grouped under per-feature capacity caps, cross-group weights zeroed, then fine-tuned under the masks---after which the additive decomposition holds to machine precision. On California Housing this recovers test MSE $0.395\pm0.006$ over three seeds, matching a NAM trained from scratch ($0.388\pm0.003$, Table~\ref{tab:multi}): pruning is violent in parameter space but benign in function space, because the function being re-encoded is already additive. The contrast that makes this a finding rather than a recipe is our weight-space alternative, WTA decay, which constrains \emph{parameters} during training (non-winning weights decay even while carrying signal) and plateaus at $0.439$ (seed-matched single run) after the identical pruning protocol: when the target is a functional property, constraining behavior first and structure second dominates constraining the weights directly. Two silent failure modes surfaced here---Adam's stale moments moving pruned weights off their masks (gap inflating from $10^{-6}$ to $3\times10^{-2}$; fix: re-project after every step), and unconstrained winner-take-all grouping collapsing onto a single feature (fix: capacity caps with a minimum quota)---and are dissected in Section~\ref{sec:failures}. We emphasize the demotion implied by this subsection's title: crystallization is an optional exactness post-processing, not a modeling route, and we exclude it---together with its weight-space ablation, WTA decay, which exists only as the comparator for the behavior-first claim---from the cross-dataset benchmark accordingly; on small datasets the multi-stage pipeline is fragile, and part of the deficit we initially attributed to it traced, on re-examination, to a batch-size protocol artifact of the kind Section~\ref{sec:multidataset} documents for Boston.

\subsection{Classical Backfitting with Neural Smoothers}
\label{sec:backfit}

A classical baseline predating neural additive modeling by decades calibrates what joint fitting buys.

\paragraph{Backfitting.} The naive alternative to joint fitting---train $d$ single-feature networks independently on $y$ and stack them linearly---fails badly ($0.499$, barely above linear regression) despite living in the same additive class, because each branch learns the \emph{marginal} $\E[y|x_i]$, which under correlated features double-counts effects that the additive class needs as \emph{partials}; the fitted stacking weights betray it with large negative coefficients subtracting the overlap. The classical repair is the backfitting algorithm of \citet{hastie1986gam}: cyclically refit each component on the partial residual $y - \alpha - \sum_{j\ne i} f_j(x_j)$, with mean-centering of each component for identifiability. Instantiated with the same single-feature networks as smoothers (8 inner epochs per component per sweep, warm-started across sweeps), backfitting descends monotonically from the stacking level, crosses the from-scratch NAM at roughly sweep 20, and plateaus at $0.380\pm0.006$ over three seeds (Table~\ref{tab:multi}). The trajectory is exactly the theory: partial residuals progressively reassign shared signal among correlated features, converging to the joint additive fit---backfitting is coordinate descent in the space of additive functions, and each sweep provably decreases the training objective because updating $f_i$ leaves every other component \emph{exactly} fixed. We do not read the final $0.380$ vs.\ $0.388$ margin as a systematic advantage over the NAM: the backfitting ensemble has roughly $8\times$ the first-layer parameters per feature, and the difference is within seed and tuning noise. The robust finding is the shared plateau: masked NAM, penalty-then-crystallize, and backfitting---three entirely different procedures---all land at ${\approx}0.39$, identifying the additive limit of the dataset itself.

\subsection{A Single Network as Its Own Additive Model}
\label{sec:sections-model}

The routes above either dedicate parameters per feature (NAM, backfitting) or train a dense network and discard its cross-feature weights (crystallization). A third possibility uses one dense network $g_\theta:\R^d\to\R$ of ordinary architecture and \emph{defines} the additive predictor through its mean-conditioned sections:
\begin{equation}
\label{eq:secmodel}
F(x) \;=\; g_\theta(\bar{x}) \;+\; \sum_{i=1}^{d}\Big[\, g_\theta(x_i,\bar{x}_{-i}) - g_\theta(\bar{x}) \,\Big],
\end{equation}
where $\bar{x}$ is the training mean. Each bracketed term depends on a single coordinate, so $F$ is exactly additive by construction, for any $\theta$; the question is only how to fit $\theta$.

\paragraph{Sequential backfitting on shared weights diverges.} The direct transplant of backfitting---cyclically fit section $i$ on partial residuals by training $g_\theta$ on mean-masked inputs $(x_i,\bar{x}_{-i})$---reaches $0.54$ within ${\sim}15$ sweeps and then diverges catastrophically (test MSE ${>}10^{3}$ by sweep 40). The failure is structural rather than a matter of step sizes. Classical backfitting converges because its components are parameter-disjoint: refitting $f_i$ leaves every $f_j$, $j\ne i$, untouched, which is what makes the procedure a genuine coordinate descent with stationary partial residuals within each update. With shared weights, training section $i$ moves \emph{all} sections simultaneously; the partial residuals computed at the start of a sweep are stale before the sweep ends, the error is baked into the next component's regression target, and the feedback loop amplifies. The convergence guarantee of \citet{buja1989linear} dies precisely where the coordinates cease to be orthogonal in parameter space.

\paragraph{Penalize or project: the section model and the quadrilateral loss are two faces of one idea.} The two routes we introduce are more closely related than their descriptions suggest, and the relation is an identity. For $d=2$, the gap between the dense network and its section read-out is exactly a quadrilateral anchored at the mean: $g(x)-F(x)=g(x_1,x_2)-g(x_1,\bar{x}_2)-g(\bar{x}_1,x_2)+g(\bar{x}_1,\bar{x}_2)=Q(x,\bar{x})$; for general $d$ it is the sum of the anchored decomposition's terms of order two and higher. Additive functions are precisely the fixed points of the read-out ($g$ additive $\Rightarrow F[g]=g$, as noted above), and the quadrilateral statistic characterizes that fixed-point set. The two methods are then the classical soft and hard treatments of the same constraint: the quadrilateral loss \emph{penalizes} the distance from the fixed-point set, with the partner point sampled from the data so that the penalty weights interactions by the interventional measure; the section model \emph{composes with the projection}, discarding the interaction component by fiat rather than suppressing it. In the $\lambda\to\infty$ limit they share the same target---minimum MSE over the additive class, whose population optimum is the additive $L^2$ projection---but they do not coincide as procedures, and the distinction is visible in our numbers. The penalty forces the \emph{network itself} to become additive, a severe constraint on a dense parametrization; the projection leaves the network free and additivizes only the read-out. At finite width the two parametrizations of the additive class have different effective capacity, and the section model's empirical edge over the crystallized penalty route ($0.376$ vs.\ $0.395$ on California Housing) is precisely that capacity difference showing. The remaining asymmetries are consequences of soft versus hard, not separate designs: the projected model earns the exact guarantee everywhere but its gradients only ever flow on the cross of one-dimensional manifolds through the anchor, while the penalized model enforces $Q=0$ only on the data's swap distribution, trains on the full manifold, and buys accuracy with its small measured budget (Table~\ref{tab:multi}).

\paragraph{Joint training of the shared-section model is stable and wins.} The repair is to abandon the sequential structure altogether and minimize the joint objective $\E[(F(x)-y)^2]$ directly, backpropagating through all $d{+}1$ masked forward passes of \eqref{eq:secmodel} in every step. There are no stale residuals by construction. With the same $64{+}64$ architecture and budget as the dense baseline, this attains test MSE $\mathbf{0.376\pm0.004}$ over three seeds---the best exactly additive neural model in our study, ahead of both the from-scratch NAM ($0.388\pm0.003$) and separate-network backfitting ($0.380\pm0.006$), in an architecture-matched comparison with the dense reference (Table~\ref{tab:multi}).

Why does an exactly additive predictor built from shared weights beat one built from disjoint branches? Our working hypothesis is implicit basis sharing: the shape functions of this dataset share qualitative structure (monotone segments, saturations), and a single network can learn that structure once and reuse it across sections, whereas disjoint branches must each rediscover it from their share of the data. On this reading, \eqref{eq:secmodel} is the limiting case of basis sharing anticipated in Section~\ref{sec:landscape}: the whole network is the basis, the sections are the read-outs. The price is computational: $d{+}1$ forward passes per example, linear in $d$, which is acceptable at $d=8$ and untested at scale. A further property we note but do not yet exploit: the trained $\theta$ supports two read-out modes---the additive $F(x)$ of \eqref{eq:secmodel} and the ordinary dense forward $g_\theta(x)$---and the relationship between them (in particular, whether training through sections implicitly additivizes the dense forward pass, as the quadrilateral penalty does explicitly) is a natural follow-up experiment.

\section{Across Datasets: Rankings Are Regime-Dependent}
\label{sec:multidataset}

The single-dataset picture invites overgeneralization in both directions. We therefore repeated the converging routes on five further tabular benchmarks---Wine Quality (red), Abalone, Boston Housing, Concrete compressive strength, and daily Bike rentals---under a uniform protocol (same architectures, Adam $10^{-3}$, three seeds, 80/20 splits; batch size reduced on the smallest datasets---to 64 on Boston and Bike, 128 on Concrete---where on Boston $n_{\text{train}}=404$ would otherwise yield one gradient step per epoch and spuriously catastrophic neural results, a protocol artifact worth flagging because it initially masqueraded as method failure). We additionally include TEAM \citep{dicecco2024team,dicecco2025xai_ch6}, a Tree Ensemble Additive Model: a bagged ensemble of depth-one gradient-boosted trees. Because every stump splits on a single feature, each boosting stage---and hence the bagged ensemble---is \emph{exactly} additive by construction, in the boosted-stumps-as-GAM lineage of \citet{friedman2001greedy} and the EBM family \citep{lou2012intelligible}; its mean-conditioned marginal-effect read-out coincides with the section read-out of Eq.~\ref{eq:secmodel}, making TEAM the tree-based twin of our shared-network section model. We report TEAM at 5 bags of 1{,}000 stumps; the repository default of 100 stumps underfits noticeably (e.g.\ $0.48$ on California Housing), so capacity matters for fair comparison.

\begin{table}[t]
\centering
\small
\setlength{\tabcolsep}{4pt}
\resizebox{\textwidth}{!}{%
\begin{tabular}{lcccccc}
\toprule
& CalHousing & WineRed & Abalone & Boston & Concrete & Bike \\
& {\scriptsize $n{=}20433,d{=}8$} & {\scriptsize $n{=}1599,d{=}11$} & {\scriptsize $n{=}4177,d{=}8$} & {\scriptsize $n{=}506,d{=}13$} & {\scriptsize $n{=}1030,d{=}8$} & {\scriptsize $n{=}728,d{=}13$} \\
\midrule
Dense (non-additive) & $\mathit{0.282 \pm .002}$ & $0.534 \pm .008$ & $4.74 \pm 1.14$ & $\mathit{9.90 \pm 2.31}$ & $64.2 \pm 2.7$ & $\mathit{0.373 \pm .007}$ \\
\midrule
Quad soft, $\lambda{=}1$ (ours, near-additive) & $\mathbf{0.328 \pm .004}$ & $0.448 \pm .029$ & $\mathbf{4.40 \pm .48}$ & $\mathbf{10.65 \pm 3.08}$ & $70.4 \pm 3.7$ & $\mathbf{0.377 \pm .016}$ \\
Quad soft, $\lambda$ tuned (ours; median $\lambda^*$) & $0.301 \pm .018$ {\scriptsize(0.06)} & $0.444 \pm .034$ {\scriptsize(2.6)} & $4.38 \pm .49$ {\scriptsize(0.75)} & $10.22 \pm 2.61$ {\scriptsize(0.81)} & $71.4 \pm 4.0$ {\scriptsize(0.75)} & $0.382 \pm .016$ {\scriptsize(0.74)} \\
NAM from scratch & $0.388 \pm .003$ & $0.500 \pm .045$ & $5.05 \pm .56$ & $15.91 \pm 3.94$ & $80.1 \pm 12.2$ & $0.478 \pm .010$ \\
Backfitting & $0.380 \pm .006$ & $0.418 \pm .055$ & $4.55 \pm .44$ & $12.13 \pm 2.68$ & $48.0 \pm 4.4$ & $0.399 \pm .020$ \\
Shared sections (ours) & $0.376 \pm .004$ & $0.441 \pm .058$ & $\mathbf{4.41 \pm .51}$ & $14.45 \pm 3.20$ & $51.9 \pm 2.7$ & $0.419 \pm .013$ \\
TEAM (bagged stumps) & $0.338 \pm .006$ & $\mathbf{0.412 \pm .056}$ & $4.71 \pm .38$ & $13.27 \pm 3.61$ & $\mathbf{32.3 \pm 1.3}$ & $0.425 \pm .009$ \\
\bottomrule
\end{tabular}}
\caption{Test MSE (mean $\pm$ s.d., 3 seeds) across six tabular datasets. Bold: best model of the additive family per dataset (ties within one s.d.\ both bolded); italics: dense reference where it wins. The first quad-soft row is \emph{near}-additive---a fixed $\lambda{=}1$ for all datasets, no per-dataset tuning, with a small measured interaction residual (Table~\ref{tab:contracts})---so its comparison against the exactly additive rows trades a weaker guarantee for accuracy. The second tunes $\lambda$ per dataset by Hyperband-style successive halving on a held-out validation split (nine log-uniform candidates, $9\to3\to1$, retrained on the full training set at the selected value; parenthetical numbers are the median selected $\lambda^*$): it matches the fixed default within noise everywhere except California Housing, where it wins by selecting $\lambda^*\!\approx\!0.06$---that is, by largely abandoning additivity. Target variances: $1.33$, $0.65$, $10.4$, $84$, $279$, $3.7$. On Concrete the best additive model beats the dense network by a factor of two.}
\label{tab:multi}
\end{table}

\begin{table}[t]
\centering
\small
\setlength{\tabcolsep}{6pt}
\begin{tabular}{lllc}
\toprule
Method & Constraint acts on & Additivity guarantee & $\big\langle \overline{Q^2}/\overline{Q^2}_{\mathrm{dense}}\big\rangle_{\text{6 datasets}}$ \\
\midrule
Dense network & --- & none & $1.00$ (reference) \\
NAM & structure (masks) & exact, by architecture & $0$ \\
TEAM & structure (stumps) & exact, by architecture & $0$ \\
Backfitting & fitting algorithm & exact, by construction & $0$ \\
Shared sections (ours) & read-out & exact, by read-out & $0$ \\
Quad $\to$ crystallize (ours) & behavior, then masks & exact after pruning & ${\sim}10^{-6}$ \\
Quad soft, $\lambda{=}1$ (ours) & behavior & near-additive, \emph{measured} & $\mathbf{0.15}$ \\
\bottomrule
\end{tabular}
\caption{\textbf{Additivity contracts across methods}, with the interaction residual in a single comparable column: the quadrilateral statistic normalized by the dense model's ($\overline{Q^2}$ is in units of $y^2$, so the raw values are not comparable across datasets; the ratio is dimensionless), averaged over the six benchmark datasets. Exactly additive methods are $0$ identically and need no measurement; the soft model is the only method whose additivity is a measured quantity, and at the default $\lambda=1$ it retains on average $15\%$ of the dense interaction mass (per dataset: CalHousing $0.016$, Abalone $0.027$, Bike $0.17$, Concrete $0.21$, Boston $0.22$, Wine $0.26$), i.e.\ removes $74$--$98\%$, while paying the accuracy shown in Table~\ref{tab:multi}.}
\label{tab:contracts}
\end{table}

Table~\ref{tab:multi} also includes the penalized model itself, taken \emph{soft}: the dense network trained with the quadrilateral loss at a fixed $\lambda=1$ for every dataset---the central recommendation of the balance heuristic developed alongside the table (median $\lambda_{\mathrm{bal}}\approx1.5$ across the six datasets); we deliberately avoid per-dataset $\lambda$ selection, which without a validation protocol would leak test information. It is the strongest model of the additive family on California Housing (where it also edges TEAM), Boston, and Bike (where it matches the dense network to within noise), and tied-strongest on Abalone, at the price of a weaker contract: its additivity is near-exact and \emph{measured} rather than structural. This is also where our approach parts ways with the GA$^2$M lineage \citep{lou2013ga2m,yang2021gaminet}. GA$^2$M-style models prescribe the interaction order from outside---main effects plus an explicitly screened list of pairwise terms---so the analyst decides in advance which interactions may exist and of what order. The quad-soft model inverts this: no order is prescribed; the optimization retains whatever interaction mass the task loss defends against a global, continuously tunable budget, and the retained mass is then quantified (per coordinate, online) by the same statistic that enforced the budget. One buys structural semantics and enumerable terms; the other buys adaptivity and a certificate. The higher-order probes of Section~\ref{sec:method} suggest the two can meet---using the quadrilateral statistics to \emph{discover} which pairs a GA$^2$M should include---but we do not pursue that here.

\paragraph{Choosing $\lambda$: a dimensional argument and a balance heuristic.} Two observations remove most of the guesswork. First, the task loss and the penalty have the same units---$Q_i$ is a signed combination of function values, so $\E[Q_i^2]$ is measured in $y^2$ exactly like the MSE---hence $\lambda$ is \emph{dimensionless}, and a fixed $O(1)$ value can transfer across datasets whose target variances differ by orders of magnitude (here $0.65$ to $279$), as the sweet-spot figure confirms. Second, the natural scale within that $O(1)$ range is the \emph{balance point} $\lambda_{\mathrm{bal}} = \mathrm{MSE}(0)/\overline{Q^2}(0)$, the coefficient at which the two loss terms weigh equally for the unconstrained model: it is cheap to compute (one $\lambda=0$ run plus one pass of the statistic) and encodes the right diagnostic---a model whose interaction mass is large relative to its error is using interactions load-bearingly and warns against pushing, while a small ratio flags them as prunable. Across our six datasets $\lambda_{\mathrm{bal}}$ spans $0.18$ (CalHousing, where indeed any $\lambda>0.1$ already costs accuracy) through $1.05$ (Bike), $1.15$ (Boston), $1.9$ (Abalone), ${\approx}2.2$ (Concrete), to $4.4$ (Wine), and on every dataset with a measured sweep the optimum falls at or below the balance point, in $[0.1,1]\cdot\lambda_{\mathrm{bal}}$ (Wine $0.7\lambda_{\mathrm{bal}}$, CalHousing ${\sim}0.5\lambda_{\mathrm{bal}}$, Concrete at the low end, ${\sim}0.15\lambda_{\mathrm{bal}}$). The benchmark's fixed $\lambda=1$ is the heuristic's central recommendation (median $\lambda_{\mathrm{bal}}\approx1.5$ across the six datasets); our earliest experiments had used $\lambda=3$, the aggressive upper edge of the Wine sweet-spot plateau, and switching to the heuristic default improved or matched every dataset (California Housing most visibly, from $0.364$ to $0.328$). As an out-of-sample test of the heuristic rather than a fit to it: Bike's $\lambda_{\mathrm{bal}}=1.05$ predicts $\lambda^*\!\approx\!0.5$--$1$, and $\lambda=0.5$ indeed yields $0.375\pm0.016$, matching the dense network ($0.373$). Per-dataset heuristic selection would improve the benchmark row further still; we keep the single fixed value to preserve its no-tuning reading.

\paragraph{$\lambda$ prices a constraint; do not tune it on accuracy.} Table~\ref{tab:multi}'s tuned row makes the comparison explicit and teaches two things. Practically, accuracy-tuned $\lambda$ (Hyperband on a validation split) buys nothing over the heuristic default on the five small datasets---it matches within noise and, on two, validation noise makes it slightly worse---while the selected $\lambda^*$ track the ordering of $\lambda_{\mathrm{bal}}$. Conceptually, the one dataset where tuning wins is the cautionary tale: on California Housing the tuner selects $\lambda^*\!\approx\!0.06$ and recovers most of the dense model's accuracy by dissolving the very property the penalty exists to buy. $\lambda$ is not a nuisance parameter: it prices a constraint, and tuning it on accuracy alone amounts to asking the market how much it will pay for a constraint it is free to ignore. The answer---nothing, where interactions pay---is correct and useless. The right protocol fixes the interpretability budget first (via the heuristic, or a target residual $\overline{Q^2}$) and lets accuracy adjust, not the reverse.

For calibration, the gap between the additive plateau (${\approx}0.44$ for this width; ${\approx}0.39$ for the deeper two-layer additive models of Table~\ref{tab:multi}) and the unconstrained dense network ($0.284$ for the two-layer model) reflects genuine interaction structure in the data, dominated by the joint dependence of house value on latitude and longitude, which no additive model can capture; this is consistent with the known geography of the dataset \citep{pace1997calhousing}.

\begin{figure}[t]
\centering
\includegraphics[width=\linewidth]{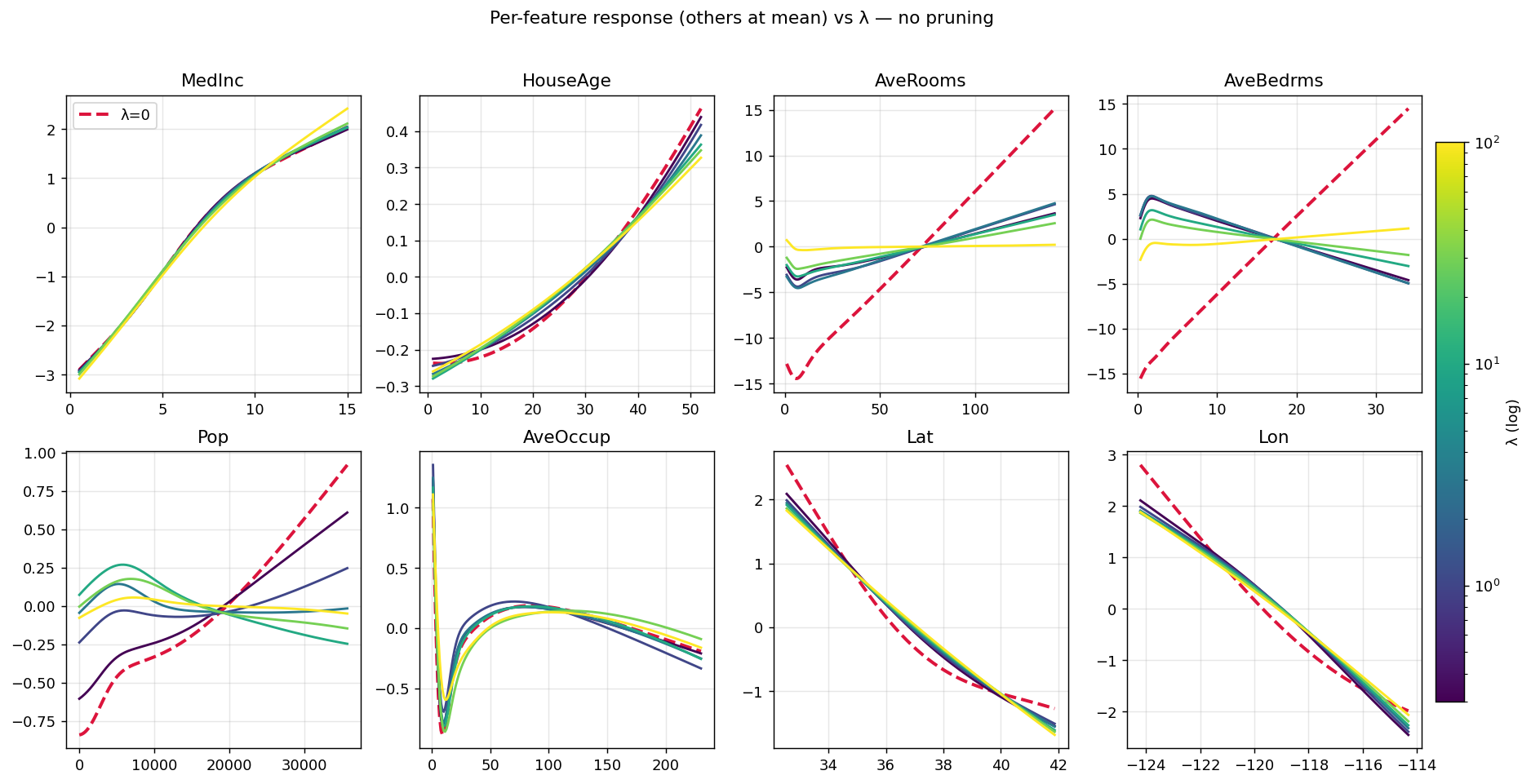}
\caption{\textbf{One-dimensional response sections as a function of $\lambda$}, without pruning: $f_\theta$ evaluated along each feature with all others held at their (standardized) mean, curves centered for comparability. Red dashed: $\lambda=0$; viridis dark$\to$light: increasing $\lambda$. Features with genuinely additive effects (MedInc, HouseAge) have $\lambda$-invariant sections; interaction-carrying features (Lat, Lon) have sections that change shape with $\lambda$ and stabilize only for $\lambda\gtrsim10$. At $\lambda=0$ these sections are conditioning-point artifacts, not feature effects.}
\label{fig:responses}
\end{figure}

Table~\ref{tab:multi} revises several single-dataset conclusions, which is precisely why it exists.

\emph{The additivity tax is not universal.} On California Housing additivity costs ${\sim}35\%$ MSE, but on three of the six datasets the best additive model \emph{beats} the dense network---most dramatically on Concrete, where TEAM halves the dense error: with $n\lesssim 4{,}000$ noisy samples the additive constraint acts as a regularizer, and where the true shape functions carry sharp thresholds (curing age, water--cement ratio) a well-matched additive learner can dominate outright. The framing ``additivity trades accuracy for interpretability,'' which our own Section~\ref{sec:sweep} adopted, is a large-$n$ statement.

\emph{Trees remain formidable on tabular data.} TEAM is the best \emph{exactly} additive model on three of six datasets---and on Concrete the best model overall by a wide margin---and is never far from the front; only the near-additive quad-soft row, with its weaker contract, edges it on the largest dataset. This is consistent with the broader pattern that gradient-boosted trees dominate tabular benchmarks \citep{grinsztajn2022trees}; the interesting refinement is that the dominance survives the restriction to depth one, i.e.\ to the exactly additive subclass. The comparison between TEAM and the shared-section network is the cleanest in the table---identical functional class, identical read-out, differing only in the function approximator---and it splits: trees win where shape functions carry sharp axis-aligned structure (California's geography, Concrete's thresholds), the network wins on Abalone's smooth monotone relationships.

\emph{Why TEAM is hard to beat.} Four ingredients, none exotic, compound. Stumps fit \emph{piecewise-constant} shape functions, which capture the thresholds, plateaus, and saturations that dominate tabular relationships---and that a smooth GELU branch must approximate with finite curvature; California's geographic effects, full of sharp administrative and coastal boundaries, are the clearest beneficiary. Split-based fitting is invariant to monotone feature transforms and therefore indifferent to the heavy tails that visibly destabilize the neural routes on the ratio features. Stagewise boosting is a functional-gradient coordinate descent that allocates its next unit of capacity to whichever feature currently has the largest residual signal---an adaptive budget the fixed per-branch widths of a NAM cannot imitate, and a one-at-a-time discipline that handles correlated features in the same spirit as backfitting. And the whole pipeline has essentially no optimization hyperparameters to mistune: no learning-rate schedule, no batch size, no seed-sensitive initialization, which is worth a great deal precisely in the small-$n$ regime where Table~\ref{tab:multi} shows the elaborate neural pipelines faltering. The flip side, visible on Abalone, is that when the true shape functions are smooth and monotone, step-function bias wastes capacity on staircase artifacts and a smooth approximator wins. A distillation probe confirms that the advantage is representational rather than informational: training our penalized network on TEAM's predictions instead of the labels performs \emph{worse} than training on the labels directly ($0.411$ vs.\ $0.401$ on California Housing), the smooth student's inability to cheaply reproduce the teacher's sharp axis-aligned steps swallowing the entire teacher margin. What TEAM knows, a network can learn from the data; how TEAM represents it, a smooth network cannot afford.

\emph{Variance dwarfs method gaps on small data.} On Boston, seed-to-seed standard deviations ($2.7$--$3.9$) exceed most pairwise method differences; the honest reading is that backfitting leads and the two crystallization routes trail, with everything else statistically entangled. Small-$n$ tabular comparisons without seed bars are not informative, ours included wherever bars overlap.

\subsection{Clinical Tabular: the Regime Where Additivity Pays}
\label{sec:clinical}

The additive lineage earned its reputation in clinical risk, so we close the benchmark with three open clinical cohorts under the identical protocol: the scikit-learn diabetes progression task (regression, $n{=}442$, $d{=}10$), Pima Indians diabetes ($n{=}768$, $d{=}8$), and UCI Heart Disease (Cleveland, $n{=}303$, $d{=}13$); the two binary cohorts are trained as proper classification: every method fits an additive---or, for the dense reference and the penalized model, unconstrained---\emph{logit} under cross-entropy, and we report test cross-entropy and AUROC.

\begin{table}[t]
\centering
\small
\setlength{\tabcolsep}{5pt}
\begin{tabular}{llcccccc}
\toprule
& & Dense & NAM & Backfitting & Sections (ours) & Quad soft $\lambda{=}1$ (ours) & TEAM \\
\midrule
DiabProg & MSE & $0.320 \pm .047$ & $\mathbf{0.284 \pm .019}$ & $0.294 \pm .023$ & $0.296 \pm .025$ & $0.292 \pm .035$ & $0.332 \pm .035$ \\
\midrule
\multirow{2}{*}{Pima} & CE & $0.565 \pm .057$ & $0.485 \pm .041$ & --- & $0.477 \pm .050$ & $0.482 \pm .038$ & $0.476 \pm .034$ \\
& AUROC & $0.812 \pm .030$ & $0.837 \pm .027$ & --- & $0.845 \pm .028$ & $0.838 \pm .027$ & $0.839 \pm .024$ \\
\midrule
\multirow{2}{*}{Heart} & CE & $0.642 \pm .262$ & $0.355 \pm .061$ & --- & $0.352 \pm .063$ & $0.354 \pm .070$ & $0.411 \pm .097$ \\
& AUROC & $0.895 \pm .040$ & $0.923 \pm .022$ & --- & $0.926 \pm .023$ & $0.924 \pm .024$ & $0.905 \pm .031$ \\
\bottomrule
\end{tabular}
\caption{Clinical cohorts. DiabProg is regression (test MSE, 3 seeds; bold: best). Pima and Heart are trained as proper classification---additive \emph{logit}, cross-entropy---and report test cross-entropy and AUROC over \emph{ten} seeds. No entry is bolded on the binary rows deliberately: the additive routes are statistically indistinguishable there (paired differences within one seed s.d.; even the nominal winner flips between seed subsets), while every one of them dominates the dense reference, whose Heart fits are additionally unstable ($\pm.26$). Classical backfitting under a logistic link is the local-scoring algorithm, which we do not implement (dashes).}
\label{tab:clinical}
\end{table}

The result is the paper's smallest-data regime taken to its conclusion: every additive route beats the dense network on every cohort and every metric (Table~\ref{tab:clinical}), and the win--win regime survives the change of loss with room to spare. At $\lambda=1$ the penalized model improves on its unconstrained counterpart on regression ($0.292$ vs.\ $0.310$ MSE, $\lambda_{\mathrm{bal}}=6.6$, the largest in the study) and, on the link scale, in classification---cross-entropy $0.482$ vs.\ $0.565$ on Pima and $0.354$ vs.\ $0.642$ on Heart over ten seeds, AUROC rising in both cases---while retaining a fraction of a percent of the logit-scale interaction mass. Within the additive family the binary rows carry no bold deliberately: at ten seeds the additive routes are statistically tied---paired cross-entropy differences of $0.003$--$0.008$ against seed spreads several times larger, win counts of $6$--$7$ out of $10$---and the tie is robust in an instructive way: TEAM, last on Pima in an earlier three-seed run, is nominally \emph{first} there at ten, the kind of winner-flip that is exactly what a tie looks like through a small-sample keyhole. What survives every replication is the regime claim, not a ranking: at a few hundred samples every additive route dominates the dense reference, and there are too few samples to separate the additive routes from one another. \paragraph{The penalty lives on the link scale.} Nothing in the construction is tied to squared error: for classification the network outputs a logit $z_\theta(x)$, the task loss is cross-entropy, and the quadrilateral penalty is applied to $z_\theta$ directly---driving it to zero makes the \emph{logit} additive, i.e.\ the model a logistic GAM, exactly as classical additive modeling prescribes for generalized links; the same recipe extends to any GLM link. Table~\ref{tab:clinical}'s binary rows are produced this way, and the penalized logit removes $99.7\%$ of the logit-scale interaction mass on both cohorts while \emph{improving} both metrics.

These are open UCI-class cohorts rather than credentialed ICU data; the MIMIC-scale study that the asthma episode of \citet{caruana2015intelligible} calls for---and that Bayesian NAMs have begun \citep{bouchiat2024lanam}, with credible intervals on MIMIC-III mortality shape functions---is the dedicated follow-up flagged in Section~\ref{sec:limits}, where surrender curves would audit which clinical interactions a model refuses to give up.

\section{Do the Routes Agree on the Functions?}
\label{sec:agreement}

Matching test MSE does not imply matching models: additive decompositions are identifiable only up to constants under a product measure, and different procedures could in principle attain similar accuracy with different shape functions when features are correlated. Figure~\ref{fig:allshapes} overlays the centered shape functions of the benchmark's four additive routes together with the penalized model's mean-conditioned sections at four values of $\lambda$.

\begin{figure}[t]
\centering
\includegraphics[width=\linewidth]{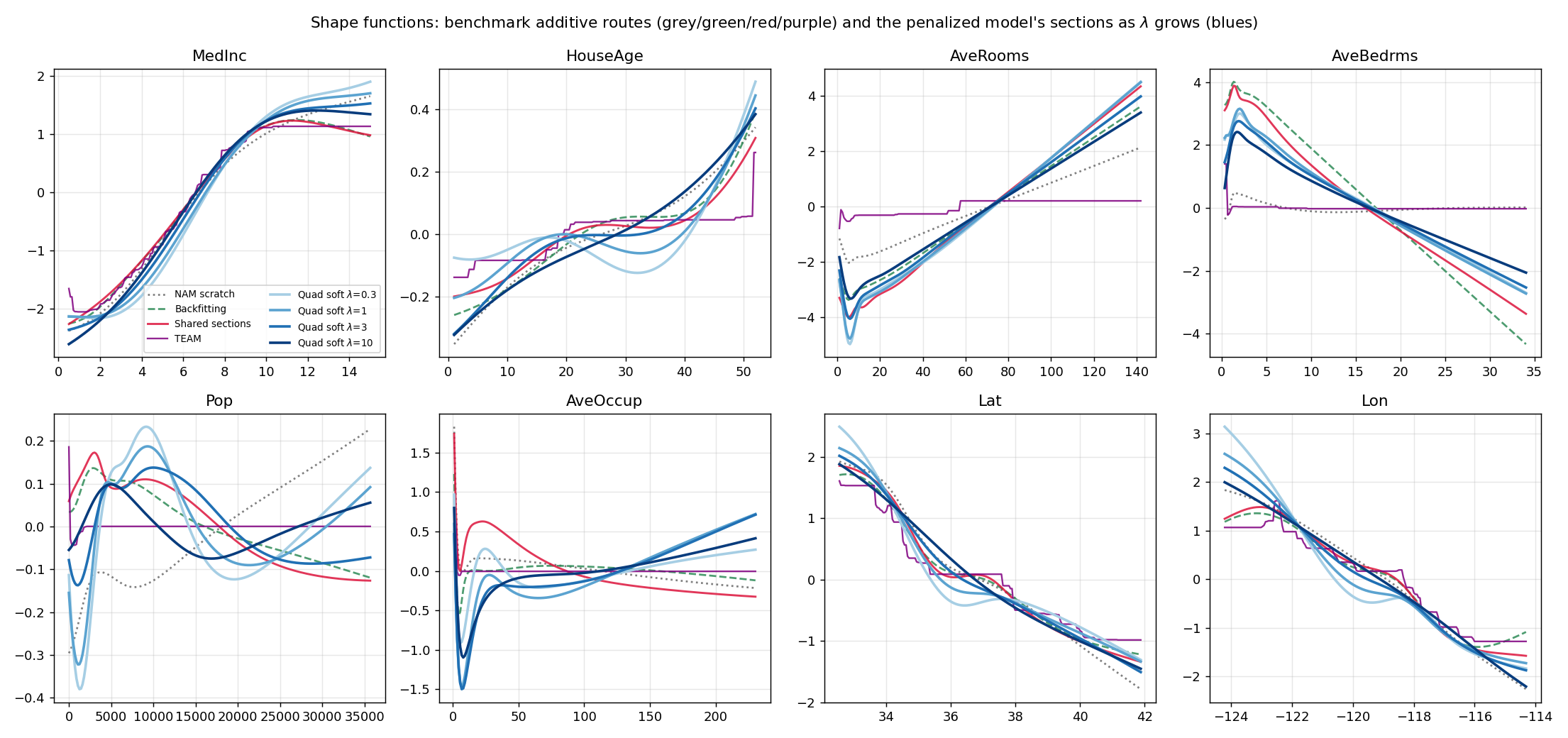}
\caption{\textbf{Agreement and convergence in one picture}, centered per feature. Four procedures with disjoint mechanisms---fixed masks (NAM, dotted grey), cyclic partial-residual refitting (backfitting, dashed green), the shared-section model (red), and bagged boosted stumps (TEAM, purple)---recover mutually consistent functions on the high-density regions of every feature, diverging only in sparse tails where the data cannot constrain them. The blue gradient overlays the \emph{penalized dense model's} sections at $\lambda=0.3,1,3,10$: at small $\lambda$ residual interactions leak into the sections and pull them off the common curves, and as $\lambda$ grows the sections collapse onto the consensus---the $\lambda$-convergence of Section~\ref{sec:sections} rendered directly against the independent routes that define the target.}
\label{fig:allshapes}
\end{figure}

The agreement is striking given how different the mechanisms are---TEAM's piecewise-constant staircases trace the same profiles the smooth neural fits draw, despite sharing no machinery with them. On median income all five curves coincide over the full range; on house age, latitude, and longitude they coincide over the data-dense regions, including the nontrivial multi-modal profiles of the two geographic coordinates---which, being the one-dimensional surrogates of a fundamentally two-dimensional effect, had no a priori reason to be stable across procedures. Disagreement concentrates exactly where it should: in the extreme tails of the heavy-tailed ratio features, where standardized values beyond a few $\sigma$ correspond to a handful of block groups and the components are unconstrained by data. The WTA route (the weakest by MSE among the six) is also the visible outlier where curves separate, consistent with its soft phase distorting components that fine-tuning cannot fully repair.

We read this as evidence that on this dataset the additive projection of the regression function is a well-identified object: the plateau at MSE ${\approx}0.38$--$0.40$ is not a family of accuracy-equivalent but functionally distinct solutions, but a single solution reached from five directions. A practical corollary: when two independent additive-fitting procedures disagree on a shape function in a data-dense region, at least one has not converged.

\paragraph{Identifiability under correlated features.} The tail disagreements deserve a sharper diagnosis than data scarcity alone, because a structural issue lurks beneath them. When features are strongly dependent---average rooms and average bedrooms here, with correlation above $0.8$---the additive \emph{sum} $f_i(x_i)+f_j(x_j)$ is well determined on the data manifold, but the individual components are not: adding $g(x_i)$ to one and subtracting its best predictor $\hat{g}(x_j)$ from the other changes each shape function while leaving predictions nearly unchanged wherever the two features co-vary. This is concurvity, the nonlinear analogue of collinearity, and it means that in strongly correlated directions our agreement figure shows five estimates of an under-determined quantity---their consistency in dense regions reflects shared inductive biases toward smoothness as much as identification by the data. Two families of remedies exist. Decompositions can be \emph{defined} to be identifiable under dependence: the weighted functional ANOVA of \citet{hooker2007anova} imposes orthogonality relative to the joint distribution, accumulated local effects \citep{apley2020ale} build components from local derivatives so that correlated marginals do not leak into one another, and purification \citep{lengerich2020purifying} post-hoc moves mass between components into canonical form. Alternatively, regularization can select among the equivalent decompositions during training: the concurvity penalty of \citet{siems2023concurvity} shrinks the correlation between fitted components, pinning down the representative with least mutual cancellation. The two regularizers compose naturally---the quadrilateral penalty controls interactions \emph{within} the model while the concurvity penalty controls redundancy \emph{between} its additive components---and their combination, quadrilateral for additivity plus concurvity for identifiability, is in our view the most promising follow-up experiment this paper does not run.

\section{Anatomy of the Failure Modes}
\label{sec:failures}

Three qualitatively different failures surfaced across the routes, and juxtaposing them reveals a common structure. Adam's stale moments moved pruned weights off their zeros, silently inflating the additivity gap by four orders of magnitude (Section~\ref{sec:exact}); unconstrained winner-take-all group assignment collapsed onto a single feature, killing the branches for median income and both geographic coordinates (Section~\ref{sec:exact}); and sequential backfitting on shared weights diverged because updating one section moved all the others, invalidating the partial residuals mid-sweep (Section~\ref{sec:sections-model}).

In each case, a guarantee was imported from a setting whose preconditions the new setting quietly violated. Mask-based sparsity assumes the \emph{entire} update rule respects the mask, but Adam's state is part of the update rule and predates the mask. Winner-take-all assignment inherits its intuition from balanced clustering, but nothing in arg-max dynamics enforces balance, and rich-get-richer feedback does the rest. Backfitting's convergence proof \citep{buja1989linear} assumes parameter-disjoint components, and weight sharing removes exactly that hypothesis. None of the three failures announces itself as a violated precondition: the first manifests as a functional property silently degrading while accuracy looks fine, the second as a model inexplicably worse than linear regression, the third as training that improves for fifteen sweeps and then explodes.

The practical lesson we take is that constraints on learned functions must be audited against the \emph{whole} training system---optimizer state, assignment dynamics, staleness of cached quantities---not just against the loss and the architecture. Each fix was cheap once the failure was understood (re-project after every step; cap group capacities; replace sequential residual fitting with a joint objective), and each failure is generic enough that we expect it to be present, and silent, in existing pipelines that combine masking, learned grouping, or alternating optimization with adaptive optimizers.

\section{Limitations}
\label{sec:limits}
\label{sec:limits}

The $\lambda$-sweep, section analysis, and crystallization comparison are conducted on California Housing only; Section~\ref{sec:multidataset} extends the converging routes to three further datasets and indeed overturns two single-dataset readings (the universality of the additivity tax, and the competitiveness of the crystallization pipeline), but six small tabular benchmarks remain a narrow base; the $\lambda$-sweep itself is single-dataset, while the surrender-curve analysis has now been replicated over five seeds on all four datasets (Section~\ref{sec:surrender}), which is what exposed the fragility of its initial single-seed reading. The penalty is enforced on the data manifold, so no claim is made about interactions off-distribution, and the stochastic estimator has a noise floor (${\sim}10^{-4}$ here) below which it cannot certify additivity---exactness always requires crystallization. Intermediate $\lambda$ yields behaviorally near-additive but structurally unconstrained models, which is a feature for the trade-off analysis and a bug for anyone hoping the knob traverses a nested family of interaction orders; the block-restricted variant of the penalty, and its comparison against GA$^2$M-style baselines with a learned (Lat, Lon) block, is the natural next experiment and is not performed here. The shared-net section model leads the neural additive routes on two of four datasets, but the implicit-basis-sharing hypothesis we offer for its advantage has not been isolated by ablation, and its $d{+}1$ forward passes per example are untested beyond $d=13$. The KAM analogy of Section~\ref{sec:surrender} is offered as a reframing, not a result. Finally, the per-step cost of the penalty is four forward passes times $|S|$ sampled coordinates, roughly tripling training cost at our settings; scaling behavior in $d$ is untested.

The natural next deployment is clinical tabular risk---MIMIC-style ICU mortality and readmission cohorts, where the additive lineage earned its reputation and where a measured interaction budget has audit value---which requires credentialed data access we deliberately leave to a dedicated study.

\section{Conclusion}

We compared routes from dense networks to additive predictors under a fixed budget, on six tabular datasets. A four-point finite-difference statistic serves throughout as instrument, penalty, and certificate: it turns additivity into a continuous training-time dial, reveals that most learned interactions are redundant and removable almost for free, and---because it decomposes by coordinate---exposes interaction as a quantity that evolves under optimization pressure. Watching that evolution taught us something we did not expect and then, on replication, corrected our first reading of it: which interactions survive an additivity regularizer is only weakly and unstably predicted by how large they were, so post-hoc interaction rankings are unreliable guides to what a regularized model will keep, and only online measurement on the actual run reveals it. The route comparison delivered its own reversals. Behavior-first-then-crystallize dominates weight-space constraints, classical backfitting with neural smoothers quietly matches modern per-feature architectures, one ordinary dense network read out through its own mean-conditioned sections---the shared-section model, the hard-projection counterpart of the quadrilateral penalty, sharing its $\lambda\to\infty$ target though not its parametrization---is the strongest exactly additive neural model we trained, and the same network trained soft with the penalty is stronger still where near-additivity suffices, and bagged boosted stumps---adaptive capacity allocation, transform invariance, nothing to mistune---remain the method to beat on tabular data. None of these rankings is universal: in low data, additivity itself becomes a regularizer---the quadrilateral penalty then improves accuracy and interpretability simultaneously---while multi-stage pipelines grow fragile; and the coefficient that prices the constraint should be set by the interpretability budget, not tuned on accuracy, which where interactions pay will simply buy the constraint away. Conceptually, the soft regime earns a defense that exact additivity cannot give: since perfect additive projections may demand shape functions too wild to read---the statistical shadow of Kolmogorov's superposition theorem, whose univariate representations are necessarily non-smooth, and the tension that forced ExU activations into the original NAM---a dense model with a small, \emph{measured} interaction residual and smooth sections can be the more interpretable object. And the three silent failures we hit along the way (Section~\ref{sec:failures}) taught a single lesson we suspect generalizes: each had imported a guarantee whose preconditions the new setting had quietly voided---and none of them announced it.

\bibliographystyle{unsrtnat}

\end{document}